\documentclass[conference]{IEEEtran}
\IEEEoverridecommandlockouts
\usepackage{times}
\usepackage{amsthm, amsfonts, amsmath, amssymb, mathtools}
\usepackage[numbers]{natbib}
\usepackage{multicol}
\usepackage{xcolor}
\usepackage[bookmarks=true]{hyperref}

\newcommand{\state}{x}
\newcommand{\sset}{\mathcal{X}}

\newcommand{\stateseq}{\mathbf{\state}}
\newcommand{\ctrl}{u}
\newcommand{\ctrlseq}{\mathbf{\ctrl}}
\newcommand{\cset}{\mathcal{U}}

\newcommand{\dstb}{d}

\newcommand{\tfunc}{l}

\newcommand{\vfunc}{V}

\newcommand{\vfsafe}{\vfunc_s}
\newcommand{\cost}{J}

\newcommand{\traj}{\xi}

\newcommand{\tvar}{t}
\newcommand{\tend}{T}
\newcommand{\tdummy}{\tau}

\newcommand{\ph}{h}
\newcommand{\ch}{h_c}
\newcommand{\tendmpc}{K}
\newcommand{\mpcindex}{k}
\newcommand{\dyn}{f}
\newcommand{\ddyn}{f_d}

\newtheorem{problem}{Problem}

\newtheorem{lemma}{Lemma}

\begin{document}

\title{Safe and Performant Deployment of Autonomous Systems via Model Predictive Control and Hamilton-Jacobi Reachability Analysis}

\author{Hao Wang$^{1,2}$, Armand Jordana$^{3}$, Ludovic Righetti$^{3}$, Somil Bansal$^{2}$%
\thanks{$^{1}$Author is with the Department of Electrical and Computer Engineering, University of Southern California.\ (email:\{haowwang\}@usc.edu)}
\thanks{$^{2}$Authors are with the Department of Aeronautics and Astronautics, Stanford University. (email:{\{haowwang, somil@stanford.edu\})}}%
\thanks{$^{3}$Authors are with the Electrical and Computer Engineering Department, New York University. (email: \{aj2988, ludovic.righetti\}@nyu.edu)}
\thanks{This research is supported in part by the DARPA Assured Neuro Symbolic Learning and Reasoning (ANSR) program and by the NSF CAREER program (2240163).}%
}

\maketitle

\begin{abstract}
While we have made significant algorithmic developments to enable autonomous systems to perform sophisticated tasks, it remains difficult for them to perform tasks effective and safely. Most existing approaches either fail to provide any safety assurances or substantially compromise task performance for safety. In this work, we develop a framework, based on model predictive control (MPC) and Hamilton-Jacobi (HJ) reachability, to optimize task performance for autonomous systems while respecting the safety constraints. Our framework guarantees recursive feasibility for the MPC controller, and it is scalable to high-dimensional systems. We demonstrate the effectiveness of our framework with two simulation studies using a 4D Dubins Car and a 6 Dof Kuka iiwa manipulator, and the experiments show that our framework significantly improves the safety constraints satisfaction of the systems over the baselines. 
\end{abstract}

\IEEEpeerreviewmaketitle

\section{Introduction}

Recent advances in foundation models \cite{vaswani_2017_transformer, gpt4_report} have transformed the field of robotics and rendered the prospect of robots performing useful tasks for us ever more realistic. In the field of robotics, large language models (LLM) and vision language models (VLM) have been employed to perform high level decision-making \cite{singh_2023_progprompt, kambhampati_2024_llm_cant_plan} and synthesize action controls \cite{liang_2023_code_as_policies, ahn_2022_saycan}, in diverse settings thanks to their generalizability.

Despite their impressive performance \cite{gemini_robotics, pi_0_5}, robot foundation models are prone to hallucinations and adversarial attacks \cite{huang_2025_llm_halucination_survey, kumar_2023_llm_adversarial}. Since robots are designed to deploy around humans, their actions are often safety-critical, as safety violations can lead to irreversible damages. In this work, we develop a framework to safeguard autonomous systems, including foundation model-enabled robots, against pre-defined safety violations, while optimizing for the task performance of the robots.

Under the setting that we are given a set of constraints which the system should respect, a number of methods can be used to cooptimize the performance and safety of the robot. The most prominent approach in the data-driven robot control community is converting the safety constraints into penalties and incorporate them in the reward function \cite{srinivasan_2020_safe_critic, bharadhwaj_2020_conservative_safety_critic}. Though such approaches typically generate policies that encourage safe behaviors, they do not provide any safety guarantees. On the other hand, the control community poses the problem as a state-constrained optimal control problem and solve it using dynamic programming \cite{altarovici_2013_scocp, wang_2024_coop_safe_perf}. This line of methods provides safety guarantees, but they cannot scale to high dimensional systems due to the curse of dimensionality associated with dynamic programming. Another popular family of methods is safety filtering \cite{ames_2016_cbf_qp, wabersich_2023_data_driven_filter, borquez_2024_safety_filtering}, where the nominal controller is modified minimally if it is deemed unsafe. However, safety filtering is inherently myopic given its construction and typically leads to suboptimal task performance. 

In this work, we propose to solve the problem of cooptimization using model predictive control (MPC) and Hamilton-Jacobi (HJ) reachability analysis. The core of our method is to incorporate the safety value function, obtained using HJ reachability analysis, as a final-time constraint in the MPC. Our method provides a closed-loop controller that optimizes the given performance objective while respecting pre-defined safety constraints, and our method can be efficiently scaled to high-dimensional systems and widely applicable to foundation model-enabled robots. 

The contribution of this work is two-fold: 1) we propose a scalable framework for cooptimizing safety and performance for robotic systems, and 2) we demonstrate the effectiveness of our framework in two simulation studies with a 4D Dubins Car and 6 DOF Kuka iiwa manipulator.

\section{Problem Formulation}
We formulate the problem of cooptimization of safety and performance as a hierarchical control problem of the robot. The LLM/VLM, taking language commands from users, acts as the high-level symbolic planner, and it outputs a cost function encoding the desired behaviors. The MPC controller performs low-level trajectory optimization over the cost function and pre-defined safety constraints. 

Formally, the low level trajectory optimization is formulated as a state-constrained optimal control problem \cite{altarovici_2013_scocp} in Prob. \ref{prob:state_constrained_opt_ctrl_prob}. Let us use $\traj_{\state,\tvar}^{\ctrlseq}:[\tvar,\tend]\rightarrow\sset$ to denote the state trajectory starting from state $\state$ at time $\tvar$ evolved with control signal $\ctrlseq:[\tvar,\tend)\rightarrow \cset$. With a slight abuse of the notation, we use $\traj_{\state,\tvar}^{\ctrlseq}(\tdummy)$ to denote the state at time $\tdummy\geq\tvar$ along the trajectory $\traj_{\state,\tvar}^{\ctrlseq}$.

\begin{problem}[State-Constrained Optimal Control Problem]\label{prob:state_constrained_opt_ctrl_prob}
\begin{subequations}
\begin{align}
    \begin{split}\label{eq:state_constrained_oc_cost}
    &\inf_{\ctrlseq} \quad \cost(\state, \tvar,\ctrlseq) = \int_\tvar^\tend r(\traj_{\state,\tvar}^{\ctrlseq}(\tdummy), \ctrlseq(\tdummy))  d\tdummy \\ 
    & \qquad \qquad \qquad \qquad \qquad \ \ \ + \phi(\traj_{\state,\tvar}^{\ctrlseq}(\tend)) \\
    \end{split}\\
    &s.t.  \quad \frac{d}{d\tdummy}\traj_{\state,\tvar}^{\ctrlseq}(\tdummy) = \dyn(\traj_{\state,\tvar}^{\ctrlseq}(\tdummy), \ctrlseq(\tdummy))  \ \forall \tdummy \in [\tvar,\tend) \label{eq:dyn}\\
    & \qquad \ \tfunc(\traj_{\state,\tvar}^{\ctrlseq}(\tdummy))\geq 0  \ \forall \tdummy \in [\tvar,\tend] \label{eq:state_constraint}
    \end{align}
\end{subequations}
\end{problem}

In this work, we are primarily interested in solving Prob. \ref{prob:state_constrained_opt_ctrl_prob}, as it optimizes for the desired behaviors commanded by the user, while respecting the pre-defined safety constraints. Crucially, the hierarchical structure and the set up of Prob. \ref{prob:state_constrained_opt_ctrl_prob} ensures that the robot will not violate the pre-defined safety constraints despite possible hallucinations of or even adversarial attacks on the LLM/VLM.

\section{Method}
We solve Prob. \ref{prob:state_constrained_opt_ctrl_prob} online in a model predictive control (MPC) fashion (i.e. solving Prob. \ref{prob:state_constrained_opt_ctrl_prob} over a shorter planning horizon $\ph$, executing the first few controls, and solving Prob. \ref{prob:state_constrained_opt_ctrl_prob} again from the evolved state), since the closed-loop nature of MPC provides robustness against potential modeling errors and unforeseen disturbances in deployment. 

Though optimal control solvers have demonstrated impressive speed and reliability, in practice the planning horizon $\ph$ of the MPC is much shorter than the task horizon $\tend$ due to latency requirements. While it helps to reduce the online computation burden, the short planning horizon leads to myopic behaviors, often resulting in eventual violation of the safety constraints. The key novelty of our approach is utilizing the safety value function in the MPC formulation to ensure persistent satisfaction of the safety constraints. 

\subsection{MPC Formulation}
In order to utilize the MPC framework to solve Prob. \ref{prob:state_constrained_opt_ctrl_prob}, we discretize the system dynamics and define the discrete-time state-constrained optimal control problem in Prob. \ref{prob:dt_scocp}. We denote the discrete-time dynamics, state trajectory, and control trajectory by $\ddyn$, $\stateseq$ and $\ctrlseq$, respectively. 

\begin{problem}[Discrete-Time State-Constrained Optimal Control Problem]\label{prob:dt_scocp}
    \begin{subequations}
    \begin{align}
        &\min_{\ctrlseq} \quad \cost(\state,\ctrlseq) = \sum_{\mpcindex=1}^{\tendmpc} r(\stateseq(\mpcindex), \ctrlseq(\mpcindex)) + \phi(\stateseq(\tendmpc))\label{eq:state_constrained_oc_cost} \\ 
        & s.t. \quad \stateseq(\mpcindex+1)  = \ddyn(\stateseq(\mpcindex), \ctrlseq(\mpcindex))  \ \forall \mpcindex \in \{1, 2,\ldots,\tendmpc-1\} \label{eq:dt_dynamics}\\
        & \qquad \ \tfunc(\stateseq(\mpcindex))\geq 0  \ \forall \mpcindex \in \{1, 2,\ldots,\tendmpc\} 
        \end{align}
    \end{subequations}
\end{problem}

Suppose the planning horizon is given by $\ph$. Then at time step $j$, the MPC is formulated as follows. 

\begin{problem}[MPC]\label{prob:mpc}
    \begin{subequations}
    \begin{align}
        &\min_{\ctrlseq} \quad \sum_{\mpcindex=j}^{j +\ph} r(\stateseq(\mpcindex), \ctrlseq(\mpcindex)) + \phi(\stateseq(j + \ph)) \\ \label{eq:mpc_objective}
        \begin{split}
             & s.t. \quad \stateseq(\mpcindex+1)  = \ddyn(\stateseq(\mpcindex), \ctrlseq(\mpcindex))  \\
             & \qquad \qquad \quad \forall \mpcindex \in \{j, j + 1 ,\ldots,j + \ph -1\} \\
        \end{split} \\
        & \qquad \ \tfunc(\stateseq(\mpcindex))\geq 0  \ \forall \mpcindex \in \{j, j + 1 ,\ldots, j + \ph\} 
        \end{align}
    \end{subequations}
\end{problem}

\subsection{Safety Value MPC}
Due to the fact that the planning horizon is typically shorter than the task horizon, MPC is not guaranteed to be recursively feasible. When recursive feasibility does not hold, Prob. \ref{prob:mpc} becomes infeasible at some time step $j$, despite starting from a feasible state $\state$. The most common approach to overcome this challenge is imposing a final-time constraint to ensure the system arrive at a set of states that are known to be recursively feasible. A desired final-time constraint has the following properties: 1) when satisfied, the system is guaranteed recursive feasibility, and 2) it should captures as ``many" recursively feasible states as possible. In this subsection, we introduce how we construct the desired final-time constraint using Hamilton-Jacobi reachability analysis. 

Given the state constraint in continuous-time Eq. \eqref{eq:state_constraint}, we formulate the safety optimal control problem in Prob. \ref{prob:safety_ocp}.

\begin{problem}[Safety Optimal Control Problem] \label{prob:safety_ocp}
\begin{equation}
\begin{split}
     \sup_{\ctrlseq}& \quad J(\state, \tvar, \ctrlseq) = \min_{\tdummy \in [\tvar,\tend]} \tfunc(\traj_{\state,\tvar}^{\ctrlseq}(\tdummy)) \\
    s.t.    & \quad \frac{d}{d\tdummy}\traj_{\state,\tvar}^{\ctrlseq}(\tdummy) = \dyn(\traj_{\state,\tvar}^{\ctrlseq}(\tdummy), \ctrlseq(\tdummy))  \ \forall \tdummy \in [\tvar,\tend) \\
    \end{split}
\end{equation}
\end{problem}

The value function of Prob. \ref{prob:safety_ocp}, is given by 
\begin{equation}
    \vfsafe(\state, \tvar) = \sup_{\ctrlseq} \min_{\tdummy \in [\tvar,\tend]} \tfunc(\traj_{\state,\tvar}^{\ctrlseq,\mathbf{\dstb}}(\tdummy))
\end{equation}
In this work we consider the \emph{converged} value function $\vfsafe(\state) = \lim_{\tvar \rightarrow \infty} \vfsafe(\state, \tvar)$, and we refer to this converged value function of Prob. \ref{prob:safety_ocp} as the \emph{safety value function}. It is important to note that the super 0-level set of the safety value function $\{\state\in\sset | \vfsafe(\state) \geq 0\}$ is the set of recursively feasible states considering the state constraint Eq. \eqref{eq:state_constraint} under the system dynamics Eq. \eqref{eq:dyn}. Furthermore, the super 0-level set contains all possible recursively feasible states, shown in \cite{borquez_2024_safety_filtering}.

Equipped with the safety value function, we incorporate it as a final-time constraint in the MPC formulation in Prob. \ref{prob:sp_mpc}. The MPC formulation is recursively feasible, and this result is shown in Lemma \ref{lemma:recursive_feasibility}. We implement the MPC controller using optimal control library Crocoddyl \cite{mastalli_2020_crocoddyl} and solver \cite{jordana_2023_mim_solvers}. 

\begin{problem}[Safety Value MPC]\label{prob:sp_mpc}
    \begin{subequations}
    \begin{align}
        &\min_{\ctrlseq} \quad \sum_{\mpcindex=j}^{j +\ph} r(\stateseq(\mpcindex), \ctrlseq(\mpcindex)) + \phi(\stateseq(j + \ph)) \\ \label{eq:mpc_objective}
        \begin{split}
             & s.t. \quad \stateseq(\mpcindex+1)  = \ddyn(\stateseq(\mpcindex), \ctrlseq(\mpcindex))  \\
             & \qquad \qquad \quad \forall \mpcindex \in \{j, j + 1 ,\ldots,j + \ph -1\} \\
        \end{split} \\
        & \qquad \ \tfunc(\stateseq(\mpcindex))\geq 0  \ \forall \mpcindex \in \{j, j + 1 ,\ldots,j + \ph -1\}  \\ 
        & \qquad \ \vfsafe(\stateseq(j + \ph)) \geq 0 \label{eq:svfunc_constraint} 
        \end{align}
    \end{subequations}
\end{problem}

\begin{lemma}[Recursive Feasibility of Safety Value MPC]\label{lemma:recursive_feasibility}
    Suppose $\stateseq(0)$, the initial state, is recursively feasible (i.e. $\stateseq(0)\in \{\state\in\sset | \vfsafe(\state) \geq 0\}$). Then Prob. \ref{prob:sp_mpc} is recursively feasible. 
\end{lemma}
\begin{proof}
    Since the set $\{\state\in\sset | \vfsafe(\state) \geq 0\}$ is the set of all recursively feasible states, shown in \cite{borquez_2024_safety_filtering}, and $\stateseq(0)$ is recursively feasible, we know that all states along the solution state trajectory of Prob. \ref{prob:sp_mpc} $\stateseq(k) \ \forall k \in\{1,\ldots,h\}$ are recursively feasible. By applying the first $\ch<\ph-1$ controls from the control solution trajectory $\ctrlseq$, the system arrive at a recursively feasible state $\stateseq(\ch)$, assuming the absence of disturbances. Using the same argument, we can show that $\forall k < K$, Prob. \ref{prob:sp_mpc} is feasible for initial state $\stateseq_s(k)$, where $\stateseq_s(k)$ is a state along the system's state trajectory resulting from solving Prob. \ref{prob:sp_mpc} and applying the first $\ch$ controls of the solution control trajectory $\ctrlseq$ at each time $k$, using induction.

\end{proof}

\section{Experiments}

\subsection{4D Dubins Car}
In this case study, we simulate a 4D Dubins car with the following dynamics 
\begin{equation*}
    \begin{bmatrix}
        \dot{x} \\ \dot{y} \\ \dot{\theta} \\ \dot{v}
    \end{bmatrix} = \begin{bmatrix}
        v\cos(\theta) \\ v\sin(\theta) \\ \ctrl_1 \\ \ctrl_2
    \end{bmatrix}
\end{equation*}
where $x, y, \theta$, and $v$ are the $x-$position, $y-$position, heading, and velocity of the system. $\ctrl_1$ and $\ctrl_2$ are the controls of the system. The system is tasked with moving from a random initial state to a goal state, defined in the $x-y$ plane, while avoiding obstacles over a 2 seconds time horizon. More formally, the running cost and the terminal cost, defined in Eq. \eqref{eq:mpc_objective}, are given by $r(\state) = \phi(\state) = ||[x,y]^\top - [x_g, y_g]^\top||_2$, where $[x_g, y_g]^\top$ is the goal state. The obstacles and the goal are shown in Figure. \ref{fig:dubins4d_example}. 

\begin{figure}[h!]
    \centering
    \includegraphics[width=0.95\linewidth]{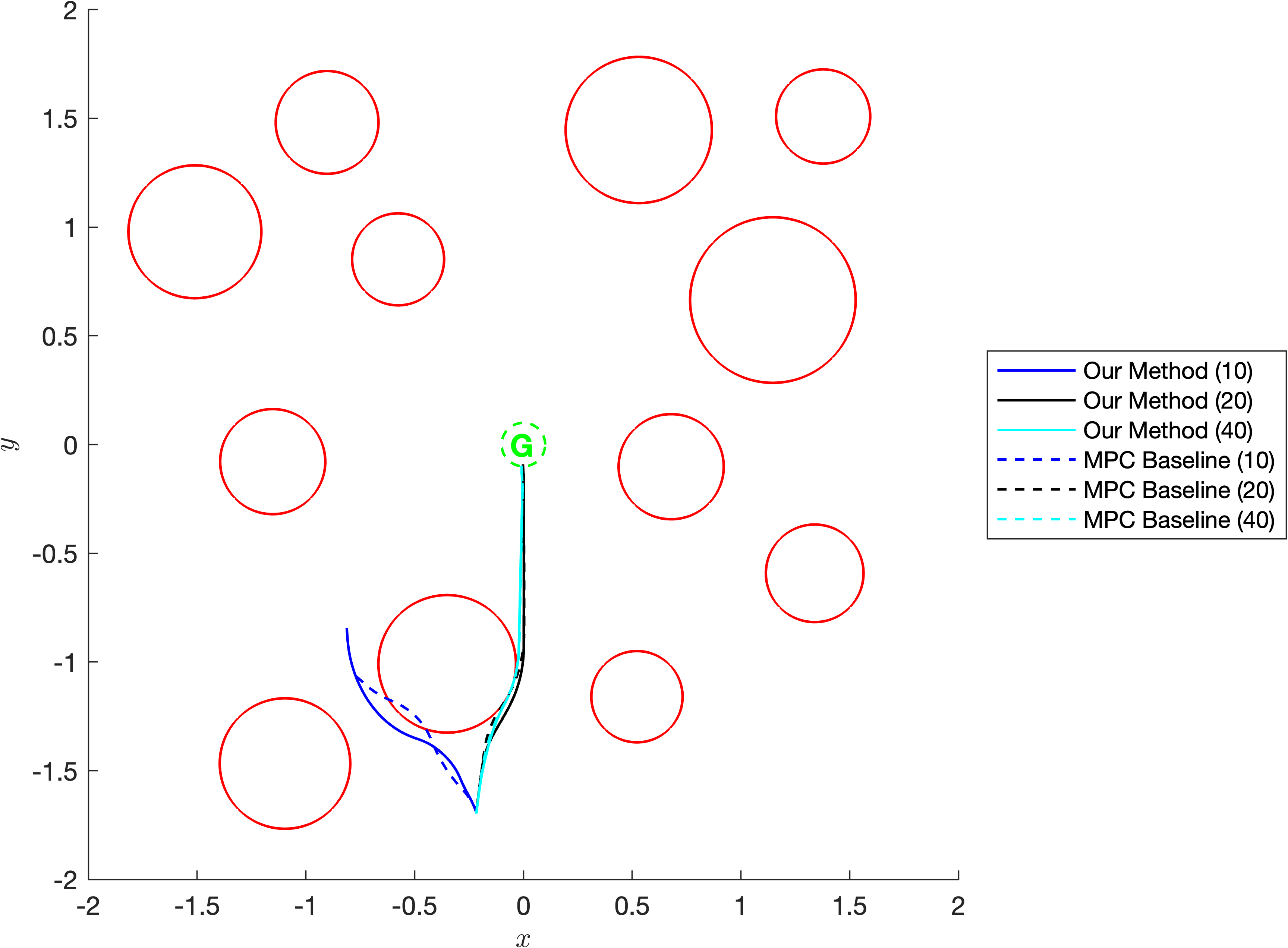}
    \caption{Dubins4D obstacle configuration and rollouts for our method and the baselines}
    \label{fig:dubins4d_example}
\end{figure}

We use a task horizon of 2 seconds and MPC time step of 0.01 second. We consider a baseline MPC controller, which is identical to our method except for not employing the safety value function as a final time constraint in Eq. \eqref{eq:svfunc_constraint}. We also vary the planning horizon (10, 20, and 40 steps) to demonstrate its effect on the safety and performance of the system. In this case study, the safety value function is computed using the LevelSetToolBox \cite{mitchell_2007_lst} and HelperOC \cite{helperOC} over a $50\times 50 \times 50 \times 30$ grid. While we synthesize the MPC controllers online, we compute the safety value function offline. 

For evaluation, we focus on the rollout success rate and the cost of the synthesized state trajectories. For clarity, we compare the results to our method with 20 planning steps. As shown in TABLE \ref{tab:dubins4d}, our method is consistently safe (with the exception of 1 trial) regardless of the planning horizon. On the other hand, the rollout success rate increases with the planning horizon. This is expected since with longer planning horizon, the MPC is able to reason about the safety constraint further into the future, and hence reducing the number of safety violations. Our method essentially condenses all the reasoning regarding the safety constraint into the safety value function, and as a result, it satisfies the safety constraint given any planning horizon. 

Consistent satisfaction of the safety constraints comes at a slight cost of the task performance. As we can see in TABLE \ref{tab:dubins4d}, our method with 20 planning steps is only able to obtain better task performance on $45.98\%$ of the trials than the MPC baseline with the same planning horizon. This is also expected because we employ the \emph{converged} safety value function in our formulation, and it typically leads to slight more conservative behaviors. Also unsurprisingly, the task performance of our method and the MPC baseline improves as the planning horizon increases. 

\begin{table}[h!]
\centering
\caption{Comparison of metrics for our method and the baselines}
\begin{tabular}{|c|c|c|}
\hline
\multicolumn{1}{|l|}{}                                      & \begin{tabular}[c]{@{}c@{}}Rollout Success\\ Rate\end{tabular} & \begin{tabular}[c]{@{}c@{}}\% Trajectories with Higher\\ Cost Compared to \\ Our Method (20)\end{tabular} \\ \hline
\begin{tabular}[c]{@{}c@{}}Our Method\\ (10)\end{tabular}   & 100\%                                                          & 93.62\%                                                                                                   \\ \hline
\begin{tabular}[c]{@{}c@{}}\textbf{Our Method}\\ \textbf{(20)}\end{tabular}   & 100\%                                                          & -                                                                                                         \\ \hline
\begin{tabular}[c]{@{}c@{}}Our Method\\ (40)\end{tabular}   & 99\%                                                           & 9.68\%                                                                                                    \\ \hline
\begin{tabular}[c]{@{}c@{}}MPC Baseline\\ (10)\end{tabular} & 80\%                                                           & 89.47\%                                                                                                   \\ \hline
\begin{tabular}[c]{@{}c@{}}MPC Baseline\\ (20)\end{tabular} & 90\%                                                           & 45.98\%                                                                                                   \\ \hline
\begin{tabular}[c]{@{}c@{}}MPC Baseline\\ (40)\end{tabular} & 98\%                                                           & 6.45\%                                                                                                    \\ \hline
\end{tabular}
\label{tab:dubins4d}
\end{table}

\subsection{6 DoF Kuka iiwa Manipulator}
Consider a scenario where the manipulator is tasked with moving a cup of coffee quickly through a cluttered environment without spilling. Though motion planning can enable the manipulator to perform agile obstacle avoidance maneuvers, it cannot do so while considering the dynamic effects of the manipulator on the coffee, likely leading to spills or running into the obstacles. In this case study, we aim to synthesize the dynamic, agile, and safe behaviors that are required to perform the prescribed task. We simulate a 6 DOF Kuka iiwa manipulator with a 12D state space (6 joint position variables and 6 joint velocity variables) and 6D control space (6 joint torques). Note that the wrist joint is locked since it is irrelevant for this experiment. The task of the manipulator is moving from the initial state to the goal end-effector position in the operation space, while avoiding obstacles in the operation space and keeping the joint accelerations within certain intervals. The joint acceleration constraints are in placed to prevent the manipulator from moving too rapidly and spilling the coffee. 

Let us denote the joint position, velocity, and acceleration by $q, v$ and $a$, and we use the shorthand $FK(\cdot)$ for forward kinematics, mapping joint positions to $(x,y,z)$ positions of the end-effector. We define the failure set to be a cylindrical region in the operation space, and the end-effector of the manipulator should avoid entering into this failure set. More concretely, the running cost and final cost are given by $r(q, v, u) = \phi(q, v, u) = ||FK(q) - x_g||_2$, the obstacle constraint is given by $l(FW(q)) \geq 0$, and the joint acceleration constraints are given by $\underline{a_i} \leq a_i \leq \overline{a_i} \ \forall i \in \{1,\ldots, 6\}$, where $x_g$ is the goal end-effector position, $\underline{a_i}$ and $\overline{a_i}$ are the lower and upper bounds for the $i^{th}$ joint acceleration.

Given the dimensionality of the system, we use a learning-based approach \cite{bansal_2021_deepreach, feng_2025_deepreach_mpc} to compute the safety value function. For this experiment, we use task horizon of 1 second, planning horizon of 15 steps, and MPC time step of 0.01 second. We again consider the MPC baseline where the safety value function constraint is not imposed. We roll out our method and the MPC baseline from 15 initial states, and our method succeeds in 11 out of 15 trials, while the MPC baseline only succeeds in 3 out of 15 trials. The end-effector trajectories of both methods for one of the trials are shown in Fig. \ref{fig:kuka12d_example}. Qualitatively, our method is more cautious around the obstacle, as it moves away from the obstacle before turning towards the goal. On the other hand, the MPC baseline heads directly for the goal and inevitably arrives at a state where it cannot avoid entering into the failure set. 

It is important to note that the MPC baseline incorporates the obstacle constraint Eq. \eqref{eq:state_constraint}. However, due to the limited planning horizon, it is unable to reason about the long-term safety and eventually leads the system into a state where it cannot satisfy the obstacle constraint. 

\begin{figure}[h!]
    \centering
    \includegraphics[width=0.95\linewidth]{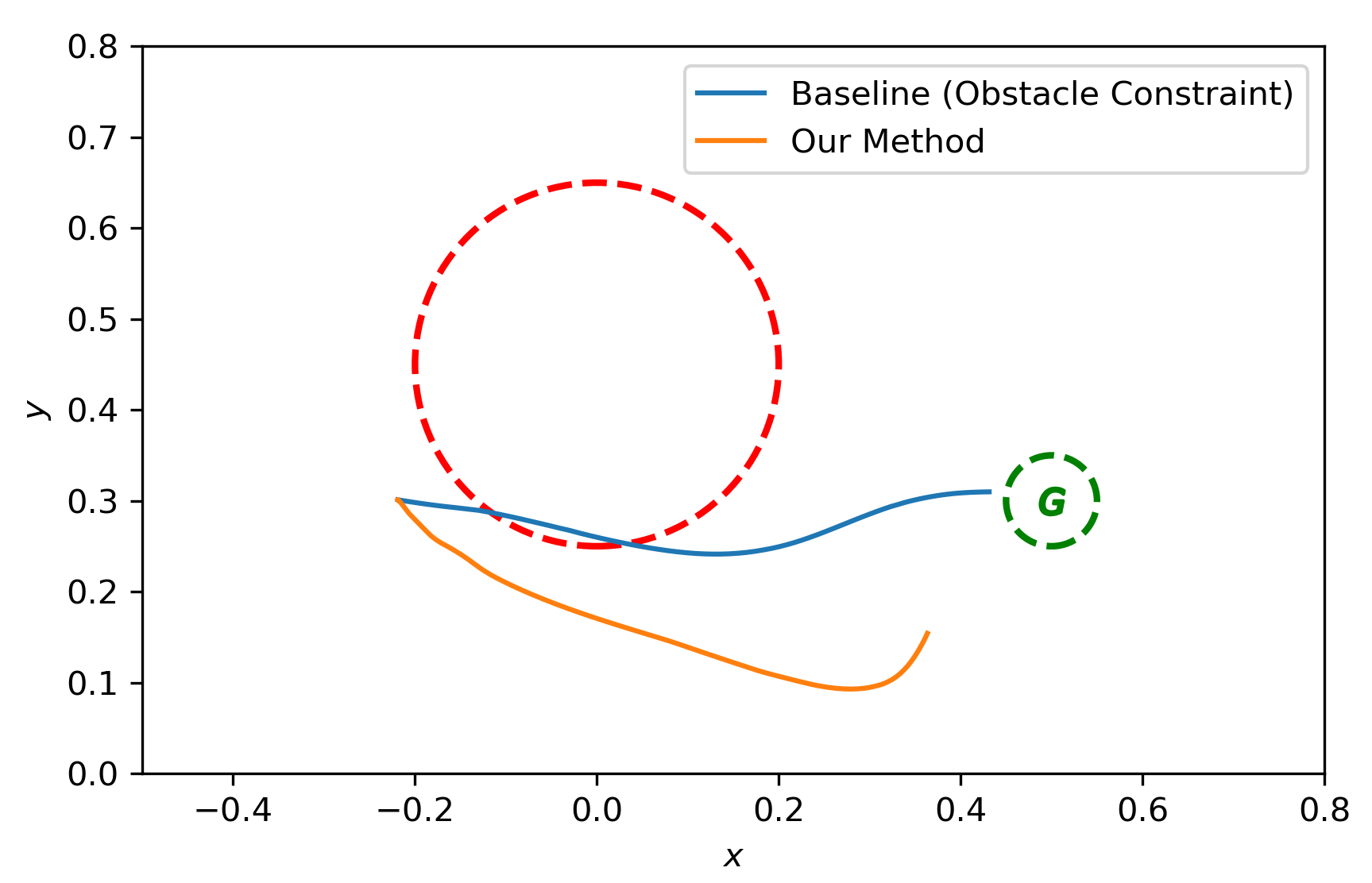}
    \caption{End-effector trajectories of our method and the MPC baseline for one of the trials}
    \label{fig:kuka12d_example}
\end{figure}

\section{Conclusion} 
\label{sec:conclusion}
In this work, we propose a scalable framework that cooptimizes safety and performance for autonomous systems, including foundation model-enabled robots. Our method incorporates the safety value function as a final-time constraint in the MPC formulation, and it is shown to improve the safety constraint satisfaction of the system in simulation studies. However, our method replies on learning-based methods for computing the safety value function for high-dimensional systems, and as a result, our framework currently cannot provide formal safety guarantees. We look to address this shortcoming in future works by utilizing verification techniques \cite{lin_2023_verify_neural_tubes, lin_2024_robust_verify_neural_tubes}.

\bibliographystyle{plainnat}
\bibliography{references}

\end{document}